\documentclass[iicol]{sn-jnl}

\usepackage{graphicx}%
\usepackage[export]{adjustbox}
\usepackage{multirow}%
\usepackage{amsmath,amssymb,amsfonts}%
\usepackage{amsthm}%
\usepackage{multirow}
\usepackage{enumitem}
\usepackage{balance}
\usepackage{subcaption}
\usepackage[linesnumbered,ruled,vlined]{algorithm2e}
\usepackage{array}
\usepackage{orcidlink}
\usepackage{lipsum}
\usepackage{xcolor}
\usepackage{soul}

\newtheorem{theorem}{Theorem}
\newtheorem{definition}{Definition}
\newtheorem{remark}{Remark}

\newcommand{\PreserveBackslash}[1]{\let\temp=\\#1\let\\=\temp}
\newcolumntype{C}[1]{>{\PreserveBackslash\centering}m{#1}}

\begin{document}

\title[Article Title]{Event-based Reconfiguration Control for Time-varying Formation of Robot Swarms in Narrow Spaces}

\author[1]{\fnm{Duy-Nam} \sur{Bui}\orcidlink{0009-0001-0837-4360}}\email{duynam@ieee.org}

\author[2]{\fnm{Manh Duong} \sur{Phung}\orcidlink{0000-0001-5247-6180}}\email{duong.phung@fulbright.edu.vn}

\author*[1]{\fnm{Hung} \sur{Pham Duy}\orcidlink{0000-0001-7878-9409}}\email{hungpd@vnu.edu.vn}

\affil[1]{\orgdiv{Faculty of Electronics and Telecommunications}, \orgname{VNU University of Engineering and Technology}, \orgaddress{\street{144 Xuan Thuy}, \city{Hanoi}, \postcode{100000}, \country{Vietnam}}}

\affil[2]{\orgdiv{Undergraduate Faculty}, \orgname{ Fulbright University Vietnam}, \orgaddress{\street{105 Ton Dat Tien}, \city{Ho Chi Minh City}, \postcode{700000}, \country{Vietnam}}}

\abstract{This study proposes an event-based reconfiguration control to navigate a robot swarm through challenging environments with narrow passages such as valleys, tunnels, and corridors. The robot swarm is modeled as an undirected graph, where each node represents a robot capable of collecting real-time data on the environment and the states of other robots in the formation. This data serves as the input for the controller to provide dynamic adjustments between the desired and straight-line configurations. The controller incorporates a set of behaviors, designed using artificial potential fields, to meet the requirements of goal-oriented motion, formation maintenance, tailgating, and collision avoidance. The stability of the formation control is guaranteed via the Lyapunov theorem. Simulation and comparison results show that the proposed controller not only successfully navigates the robot swarm through narrow spaces but also outperforms other established methods in key metrics including the success rate, heading order, speed, travel time, and energy efficiency. Software-in-the-loop tests have also been conducted to validate the controller's applicability in practical scenarios. The source code of the controller is available at \url{https://github.com/duynamrcv/erc}.
}


\maketitle

\section{Introduction}

With advancements in networked multi-agent technology, multi-robot systems (MRSs) have rapidly developed towards autonomy, offering various applications from warehouse automation to search and rescue operations~\cite{6303906,Verma2021}. A core element of these systems is a formation controller that enables the robots to collaborate in a desired configuration~\cite{Ghaderi2024,1545539,Oh2015}. In rigid formation control, achieving the desired configuration involves setting a specific target distance for each swarm agent. However, with the increasing complexity of formation tasks, the formation configuration needs to be adjustable to meet task requirements. Time-varying formation (TVF) control thus has become essential for swarm robots~\cite{Dong2015,Dong2016}.

\begin{figure}
\begin{subfigure}[b]{0.235\textwidth}
    
    \centering
    \includegraphics[width=\textwidth,frame]{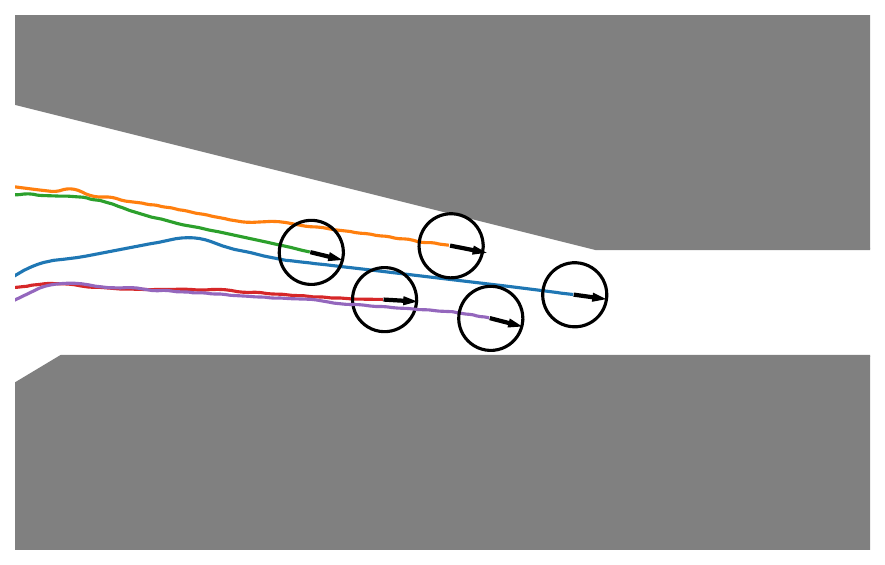}
    \caption{Rigid formation control}
    \label{fig:sample_bc}
\end{subfigure}
\begin{subfigure}[b]{0.235\textwidth}
    \centering
    \includegraphics[width=\textwidth,frame]{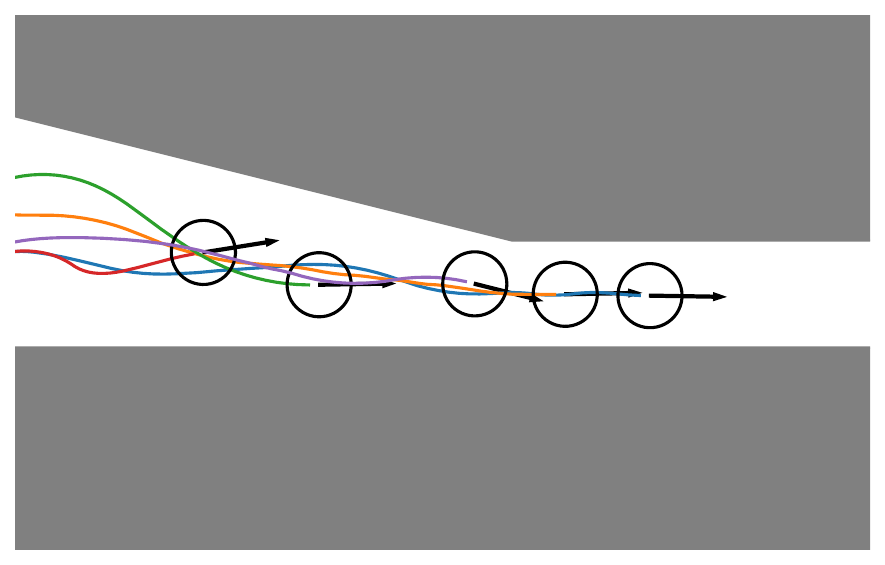}
    \caption{Reconfiguration control}
    \label{fig:sample_erc}
\end{subfigure}
\caption{A robot swarm traveling through a narrow space: (a) Motion paths of the robots using the rigid formation control~\cite{736776, Vsrhelyi2018}, which collide with surrounding obstacles; (b) Motion paths of the robots using our proposed approach, which safely navigate through the narrow space.}
\label{fig:sample}
\end{figure}

Studies in formation control show that biological swarm motion can be characterized by three primary behavioral rules: (i) \textit{cohesion}, which drives an agent toward to its neighbors; (ii) \textit{separation}, which pushes an agent away from its neighbors to prevent collisions; and (iii) \textit{alignment}, which aligns an agent to the dominant heading of its neighbors~\cite{736776,Reynolds1987,Antonelli2009}. For goal-oriented swarm motion, \textit{alignment} can be replaced by a \textit{migration} behavior so that each agent is directed in a desired orientation and moves at a preferred speed~\cite{6095129}. In complex environments, a fourth behavior called \textit{collision avoidance} can be added to guide the agent around obstacles~\cite{9565893, 9990164,1605401,10417519}. These behavioral rules are typically integrated into multi-robot systems through virtual forces within the artificial potential field (APF) to achieve coordinated movements~\cite{9981858,9561902,9990164}. Specifically, the APF has been used in~\cite{10417519,8716301,9565893} to navigate robot swarms through different environments, from open spaces to environments containing concave obstacles. In~\cite{9565893}, a fuzzy controller is included to allow the robots to avoid complex obstacles while maintaining swarm connectivity. In~\cite{Vsrhelyi2018}, evolutionary optimization is combined with a behavioral model to enable stable and decentralized navigation for large-scale aerial robot swarms in confined spaces. However, using a fixed set of behaviors limits the flexibility of the swarm formation in response to abrupt changes in the environment, especially in spaces such as caves, corridors, and tunnels, and hence increases the risk of collisions~\cite{Saska2020,AlonsoMora2017}.

Apart from the behavior-based approach, studies in formation control can be further grouped into two primary categories: \textit{rigid formation}~\cite{Saska2020,Gmez2013,Roy2018,Ebel2017,Ebrahimi2024} and \textit{adaptive formation}~\cite{Fu2020,8843165,AlonsoMora2017,1729881419862733}. In \textit{rigid formation}, the swarm shrinks or expands in size while preserving its overall shape. In~\cite{Saska2020}, model predictive control is used to generate optimal control signals for a rigid formation based on a shared map. In~\cite{Roy2018}, obstacle avoidance is included in a rigid formation through a region-based hierarchical controller that drives the robots to cohesively navigate around obstacles. Graph-based path planning and distributed model predictive control are introduced in~\cite{Ebel2017} for optimal motion when collaborating in a certain formation. While rigid formations are effective in open-space environments, they are unable to guide the robots through narrow spaces due to formation width constraints. In addition, shrinking the formation also increases the risk of collision among the robots, as illustrated in Figure~\ref{fig:sample_bc}.

In \textit{adaptive formation}, the swarm can transform among different configurations in adaptation to the environmental conditions. It is conducted by a two-layer hierarchical reconfiguration strategy that combines a virtual formation structure and behavioral control as in \cite{Fu2020}. In \cite{8843165}, particle swarm optimization is used to optimize the trajectories of each robot of the swarm so that the formation can be reconfigured for obstacle avoidance and task accomplishment. A robust adaptive formation control algorithm is also introduced in \cite{1729881419862733} in which the controller can steer the vehicles to form different shape patterns for flexible navigation. In general, adaptive formation methods are effective in guiding robots through complex environments as they allow for flexible reconfiguration in response to changing conditions. However, most existing systems rely on centralized control, and/or require communication across the entire swarm during the decision-making process. This highlights the need for a distributed formation control system with a simple but effective decision-making law that ensures safe, resilient, and efficient swarm navigation in challenging environments.

In this work, we propose an event-based reconfiguration controller (ERC) for the navigation of a decentralized robot swarm in narrow environments. The robots are equipped with local sensors and communication modules to collect information about the environment and other robots in the formation. The main contributions of this study are threefold:
\begin{enumerate}
    \item Introduce a new set of individual behaviors that meet the reconfiguration control requirements, including goal-directed motion, formation maintenance, tailgating, and collision avoidance behaviors. The behaviors are designed based on an artificial potential field making them simple to implement.
    \item Propose an event-based reconfiguration controller with two modes, \textit{``formation''} and \textit{``tailgating''}, capable of adapting the formation shape in response to environmental changes. The stability of the controller is proven using the Lyapunov theorem.
    \item Extensive simulations and comparisons have been conducted to evaluate the robustness, scalability, and effectiveness of the proposed controller. Software-in-the-loop tests have also been conducted to verify its practical applicability. The source code of the proposed controller is publicly available for further study and practical implementation.
\end{enumerate}

The remaining sections of this paper are organized as follows. Section~\ref{sec2} presents the formation model and formulation. Section~\ref{sec3} introduces the proposed event-based reconfiguration control method. Simulation, comparison, and software-in-the-loop experimental results are shown in Section~\ref{sec4}. The paper ends with conclusions drawn in Section~\ref{sec5}.
\section{Preliminaries}\label{sec2}
\subsection{Robot Model}
Consider a swarm $\mathcal{N}$ that contains $n$ robots labelled $i\in\left\{1,...,n\right\}$. The swarm is modeled as an undirected graph~\cite{10520245,Oh2011} $\mathcal{G}=\left(\mathcal{V},\mathcal{E}\right)$, where vertex set $\mathcal{V} = \left\{1,..., n\right\}$ represents the robots, and edge set $\mathcal{E}=\mathcal{V}\times \mathcal{V}$ contains robot pairs $\left(i, j\right)\in\mathcal{E}$ in which robot $i$ can communicate to robot $j$.

\begin{figure}
    \centering
    \includegraphics[width=0.35\textwidth]{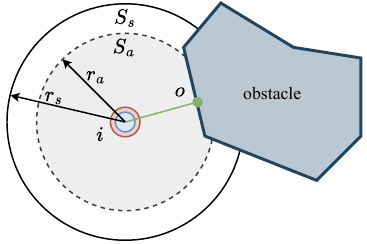}
    \caption{Illustration of a robot in the swarm having a local range sensor with sensing area $S_s$ of radius $r_s$ (solid white circle), alert area $S_a$ of radius $r_a$ (dashed gray circle), and set~$\mathcal{M}_i=\{o\}$ of the nearest data point (green dot) to an obstacle.}
    \label{fig:model}
\end{figure}

\begin{figure*}
    \centering
    \includegraphics[width=0.9\textwidth]{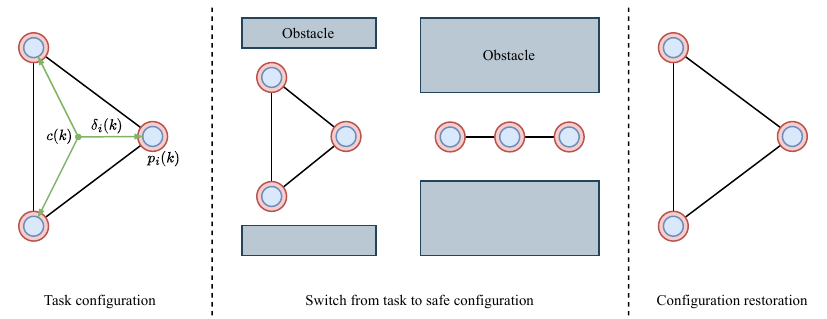}
    \caption{The proposed time-varying formation switching between the task and safe configurations to navigate through a narrow space}
    \label{fig:problem}
\end{figure*}

In this work, robots in the swarm are identical, each with a body radius $r$. Each robot is equipped with an inertial measurement unit (IMU) for position and orientation determination, a range sensor for obstacle detection, and a wireless ad-hoc network module for peer-to-peer communication with other robots. The communication delay between any two robots is assumed to be negligible~\cite{AlonsoMora2018,9527169}. The range sensor has $360^\circ$ field of view providing a scanning area $S_s$ of radius $r_s$, as depicted in Figure~\ref{fig:model}. The sensor data collected by robot $i$ at time $t(k)$ is represented as set $\mathcal{M}_i(k)$ of data points $o$ to obstacles, $\mathcal{M}_i(k) = \{o\}$. A safe area of radius $r_a$ is defined so that repulsive forces are activated when robot $i$ senses obstacles or its neighbors within this range.

The robot's dynamics is expressed in discrete time. Let $\tau$ be the time step, and $p_i(k), v_i(k), u_i(k) \in \mathbb{R}^3$  be the position, velocity, and control input of robot $i$ at time $t(k) = k\tau$, respectively. According to~\cite{Dong2015}, each robot can be modeled as a point-mass system with a double integrator as follows:

\begin{equation}
    x_i(k+1)=A_ix_i(k) + B_iu_i(k),
    \label{eqn:model}
\end{equation}
where $A_i=\left[\begin{array}{cc}I_{3\times3} & \tau I_{3\times3}\\0_{3\times3} & I_{3\times3}\end{array}\right]$ and $B_i=\left[\begin{array}{c}0_{3\times3}\\\tau I_{3\times3}\end{array}\right]$ are system matrices, $u_i$ is input acceleration, and $x_i = [p_i,v_i]^T\in\mathbb{R}^6$ is a state vector containing position and velocity. The velocities and accelerations are bounded so that $\left\Vert v_i(k)\right\Vert\leq v_\text{max}$ and $\left\Vert u_i(k)\right\Vert\leq u_\text{max}$.

\subsection{Problem formulation}

This paper addresses the time-varying formation (TVF) control for a swarm of robots in narrow spaces defined as follows~\cite{Dong2015,Dong2016}.

\begin{definition}\label{def_tvf}
Time-varying formation:  Let $\delta(k)=\left[\delta_1(k),...,\delta_n(k)\right]^T$ be a bounded time-varying vector describing the desired formation configurations. The formation is said to achieve a TVF $\delta(k)$ if all robots in the formation satisfy
\begin{equation}
    \lim_{k\to\infty}\left( p_i(k)-\delta_i(k)-c(k)\right)=0,\quad\forall i \in\{1,...,n\},
\end{equation}
where $c(k)$ is the formation center at time $k$.
\end{definition}

\begin{definition}\label{def_pro}
Safe formation: Given the TVF defined in Definition~\ref{def_tvf}, this TVF is said to be safe for any robot $i$, with $i\in\left\{1,...,n\right\}$, in the formation if the following conditions are satisfied:
\begin{enumerate}
    \item Formation configuration
\begin{equation}
    \lim_{k\to\infty}\sum_{j=1,j\neq i}^n{\left\Vert\left(p_i(k)-\delta_i(k)\right) - \left(p_j(k)-\delta_{j}(k)\right)\right\Vert}=0
    \label{eqn:form}
\end{equation}
    \item Safe distance between robots
\begin{equation}
    \left\Vert p_i(k)-p_j(k)\right\Vert > 2r
    \label{eqn:col}
\end{equation}
for all $j\neq i, j\in\left\{1,...,n\right\}$.
    \item Safe distance from obstacles
\begin{equation}
    \left\Vert p_i(k)-o\right\Vert > r
    \label{eqn:obs}
\end{equation}
for all $o\in\mathcal{M}_i(k)$. 
\end{enumerate}
\end{definition}

\begin{remark}
According to Definition~\ref{def_tvf}, the system converges to the desired configuration when~\eqref{eqn:form} is established.
Conditions \eqref{eqn:col}~--~\eqref{eqn:obs} on the other hand ensure collision-free among the robots and between the robots and their surrounding obstacles.
\end{remark}

\subsection{Formation configurations}\label{sec:config}
Formation configurations refer to the shape that the robots form while cooperating and interacting with the environment. This work considers two following primary configurations:
\begin{enumerate}
    \item \textit{Task configuration:} This configuration represents the arrangement that the robot swarm is expected to maintain to accomplish its task. Common task configurations include shapes such as polygons or V-shapes.
    \item \textit{Safe configuration:} This configuration allows the robot to operate in narrow environments where the space width is insufficient for the robots to maintain their original configuration. It is essential for the safe operation of the swarm.
\end{enumerate}
\section{Event-based Reconfiguration Control}\label{sec3}

The proposed controller incorporates an event-triggering mechanism that enables the swarm to adapt its formation shape dynamically in response to environmental conditions for safe navigation. As illustrated in Figure~\ref{fig:problem}, the desired task configuration, denoted as $\delta^*=\left[\delta^*_1(k),...,\delta^*_n(k)\right]^T$, is initially assigned to the TVF. This configuration is adjusted whenever the robots encounter a narrow space. A scaled task configuration $\delta(k)=\kappa\delta^*$, where $\kappa\in\mathbb{R}$ is a scaling coefficient, is applied if the formation can contract. Otherwise, the swarm temporarily transitions to a safe configuration to avoid collisions. Once passing through the narrow space, the swarm restores the original task configuration to resume its operation. 

To implement this event-triggering mechanism, several behaviors are designed based on potential field force so that they can form two different control modes, \textit{``formation''} and \textit{``tailgating''}, as shown in Figure~\ref{fig:control_diagram}. The \textit{``formation''} mode maintains the task configuration, whereas the \textit{``tailgating''} mode transforms the TVF to the safe configuration, as detailed below. 

\begin{figure*}
    \centering
    \includegraphics[width=0.65\textwidth]{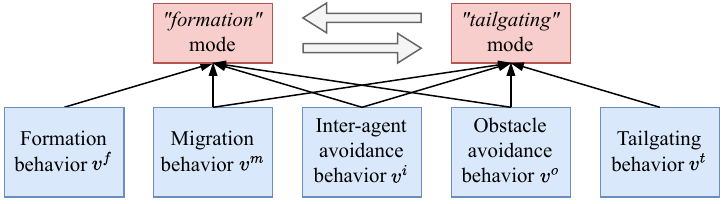}
    \caption{The proposed event-based reconfiguration control with two modes and five individual behaviors}
    \label{fig:control_diagram}
\end{figure*}

\subsection{Behavior design}
Five behaviors are designed for the controller categorized into two groups: conditionally active behaviors and continuously active behaviors. The first group includes formation and tailgating behaviors, which are alternately activated depending on the formation mode. The second group includes migration, inter-agent avoidance, and obstacle avoidance behaviors, which are always activated to maintain the formation behavior.

The formation behavior is designed based on an attractive force that moves the robots toward their desired positions. By utilizing the definition of formation configuration in (\ref{eqn:form}) and noting that $\delta(k)=\kappa\delta^*$, this behavior can be defined as follows:
\begin{equation}
    v^f_i=k_f\sum_{j=1,j\neq i}^n{\left(p_j-p_i-\kappa\left(\delta_j^*-\delta_i^*\right)\right)},
    \label{eqn:uf}
\end{equation}
where $k_f>0$ is the formation control gain.

The tailgating behavior is designed based on the relative positions among the robots in the TVF. It aims to navigate the swarm to pass through the narrow environment. When encountering a narrow environment, robot~$i$ changes its position by aligning with the nearest robot $l_i$ in front of it in the desired direction $u_\text{ref}$. Let $\mathcal{P}_i=\left\{\tilde{p}_{ij}\right\}$, where $\tilde{p}_{ij}=\left\langle (p_j-p_i),u_\text{ref}\right\rangle$, be the dot product of the vector $(p_j-p_i)$ and the desired direction $u_\text{ref}$, with $j\neq i$. The robot $j$ is determined in the front of robot $i$ when $\tilde{p}_{ij}\geq0$.
As a result, the determination of robot $l_i$ is calculated as follows:
\begin{equation}
     l_i=\begin{cases}
    \arg\min_{j}\left\{\tilde{p}_{ij}\in\mathcal{P}_i\vert\tilde{p}_{ij}\geq0\right\} & \exists~\tilde{p}_{ij}\geq0\\ 
    -1 & \text{otherwise}
     \end{cases}
    \label{eqn:li}
\end{equation}

In case the leader robot $l_i$ is determined, i.e., $l_i\neq-1$, robot $i$ will follow its leader $l_i$ and keep the desired distance $d_\text{ref}$. The attractive force toward robot $l_i$ of robot $i$ thus is expressed as follows:
\begin{equation}
    v_i^t= \begin{cases}
        k_t\left(p_{l_i}-p_i-d_\text{ref}\dfrac{v_{l_i}}{\left\Vert v_{l_i}\right\Vert}\right)+v_{l_i} & \text{if } l_i\neq-1\\
        0 & \text{if } l_i=-1
    \end{cases}
\label{eqn:ut}
\end{equation}
where $k_t$ is the tailgating control gain.


The migration behavior aims to navigate the robots toward the goal at a reference velocity. Let $v_\text{ref}$ be the preferred speed and $u_\text{ref}$ be the goal direction. The migration term for each agent is defined as
\begin{equation}
    v_i^m=v_\text{ref}u_\text{ref}.
\end{equation}
When moving through a narrow space, $u_\text{ref}$ can be adjusted to align with the environment's boundary, serving as a virtual goal direction, and then shifted to the actual goal direction once the swarm exits the narrow space~\cite{9565893}.

The inter-agent and obstacle avoidance behaviors are designed to prevent collision. Denote $\mathcal{M}_i$ as the set of points on the obstacle's boundary closest to robot $i$, as illustrated in Figure~\ref{fig:model}. Those repulsive forces are given as follows:
\begin{equation}
    v_i^i=k_{i}\sum_{j=1,j\neq i}^n{v_{ij}^i}
    \label{eqn:ui}
\end{equation}
\begin{equation}
    v_i^o=k_o\sum_{o\in\mathcal{M}_i}v_{io}^o
    \label{eqn:uo}
\end{equation}
where $k_i,k_o>0$ are the inter-agent and obstacle avoidance gains, respectively. Denote $p_{ij}=p_i-p_j$, and $\hat{p}_{ij}=\dfrac{p_{ij}}{\left\Vert p_{ij}\right\Vert}$ as the relative position and the normalized vector between robot $i$ and robot $j$, respectively. Similarly, denote $p_{io}$ and $\hat{p}_{io}$, $o\in\mathcal{M}_i$, as the relative position and the normalized vector between robot $i$ and obstacle $o$, respectively. The associated inter-agent avoidance $v_{ij}^i$ and obstacle avoidance $v_{io}^o$ are given as follows \cite{9143127}:
\begin{equation}
    v_{ij}^{i}=\begin{cases}
\left(\dfrac{1}{\left\Vert p_{ij}\right\Vert }-\dfrac{1}{r_{a}}\right)\dfrac{1}{\left\Vert p_{ij}\right\Vert ^{2}}\hat{p}_{ij}+v_{j} & \text{if }\left\Vert p_{ij}\right\Vert <r_{a}\\
0 & \text{otherwise}
\end{cases}
\end{equation}
\begin{equation}
    v_{io}^{o}=\begin{cases}
\left(\dfrac{1}{\left\Vert p_{io}\right\Vert }-\dfrac{1}{r_{a}}\right)\dfrac{1}{\left\Vert p_{io}\right\Vert ^{2}}\hat{p}_{io} & \text{if }\left\Vert p_{io}\right\Vert <r_{a}\\
0 & \text{otherwise}
\end{cases}
\end{equation}
These equations imply that the avoidance behaviors are only activated when threats are in the alert range $r_a$. The magnitude of the avoidance forces is inversely proportional to the distance between the robot and threats.

\subsection{Event-based reconfiguration control}\label{sec:erc}

The proposed event-based reconfiguration control includes two modes, \textit{``formation''} and \textit{``tailgating''}, as illustrated in Figure~\ref{fig:control_diagram}. The overall velocity applied to each robot is the sum of the velocities contributed by individual behaviors as follows: 
\begin{equation}
    \tilde{v}_i=\begin{cases}
        v_i^f+v_i^m+v_i^i+v_i^o & \text{if mode = \textit{``formation''}}\\
        v_i^t+v_i^m+v_i^i+v_i^o & \text{if mode = \textit{``tailgating''}}
    \end{cases}
    \label{eqn:v}
\end{equation}

\begin{figure}
    \centering
    \includegraphics[width=0.25\textwidth]{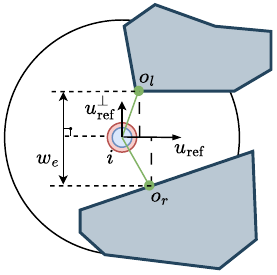}
    \caption{Environment's width estimation}
    \label{fig:we}
\end{figure}

At each time step, a robot determines its mode based on sensing data $\mathcal{M}_{i}$ collected from its local sensor. Each point in $\mathcal{M}_{i}$ is classified into the cluster either on the left or right side of the robot along the $u_\text{ref}$ direction. Let $o_l$ and $o_r$, respectively, be the points on the left and right clusters having a minimum distance to the robot. As depicted in Figure~\ref{fig:we}, the width of the environment is computed as the magnitude of the inner product between $(o_r-o_l)$ and the vector perpendicular to $u_\text{ref}$, $u^\bot_\text{ref}$, as follows:
\begin{equation}
    w_e= \left\vert\left\langle\left(o_r-o_l\right),u^\bot_\text{ref}\right\rangle\right\vert.
    \label{eqn:we}
\end{equation}
The formation's width $w_f$, on the other hand, can be determined through the predefined formation topology. Given the widths of the environment and formation, the scaling factor~$\kappa$ contributing to \eqref{eqn:uf} can be computed as follows:
\begin{equation}
    \kappa = 
    \begin{cases} 
        \dfrac{w_e - 2r}{w_f} & \text{if } w_e \geq \lambda r \\
        0 & \text{otherwise}
    \end{cases}
    \label{eqn:kappa}
\end{equation}

From (\ref{eqn:we}) and (\ref{eqn:kappa}), each robot can automatically determine its mode and compute scaling factor $\kappa$. Let $\lambda>2$ be a pre-defined threshold to switch between two modes. When the environment is too narrow, i.e. $w_e<\lambda r$, the system activates the tailgating mode; otherwise, it operates in the formation mode. The desired velocity $\tilde{v}_i$ corresponding to each mode can be obtained, as presented in Algorithm~\ref{alg:our}.
\begin{algorithm}
\caption{Pseudocode of the ERC}
\label{alg:our}
Get a data set of observed obstacle points $\mathcal{M}_i$\;
\If{$\mathcal{M}_i$ is empty}{
    mode $\leftarrow$ \textit{``formation''}\;
    $\kappa \leftarrow 1.0$\;
}
\Else{
    Determine obstacle points $o_l$, $o_r$\;
    \If{$\nexists o_l$ or $\nexists o_r$}{
        mode $\leftarrow$ \textit{``formation''}\;
        $\kappa \leftarrow 1.0$\;
    }
    \Else{
        Compute space width $w_e$\tcc*[r]{Eq. \ref{eqn:we}}
        \If{$w_e\leq\lambda r$}{
            mode $\leftarrow$ \textit{``tailgating''}\;
        }
        \Else{
            mode $\leftarrow$ \textit{``formation''}\;
            Estimate desired formation width $w_f$\;
            \If{$w_e-2r\leq w_f$}{
                Compute scaling factor $\kappa$\tcc*[r]{Eq. \ref{eqn:kappa}}
            }
            \Else{
                $\kappa\leftarrow1.0$\;
            }
        }
    }
}
Compute desired velocity $\tilde{v}_i$\tcc*[r]{Eq. \ref{eqn:v}}
\Return $\tilde{v}_i$\;
\end{algorithm}

Note from \eqref{eqn:model} that the control signal is acceleration. It can be obtained by differentiating the velocity in the discrete time as follows:
\begin{equation}
    \tilde{u}_i(k+1) =\dfrac{\tilde{v}_i(k+1)-\tilde{v}_i(k)}{\tau}.
\end{equation}
Bounding on the acceleration and velocity can be included as follows:
\begin{equation}
    u_i=\dfrac{\tilde{u}_i}{\left\Vert\tilde{u}_i\right\Vert}\min(\left\Vert\tilde{u}_i\right\Vert, u_\text{max})
\end{equation}
\begin{equation}
    v_i=\dfrac{\tilde{v}_i}{\left\Vert\tilde{v}_i\right\Vert}\min(\left\Vert\tilde{v}_i\right\Vert, v_\text{max})
\end{equation}

\subsection{Stability analysis}

\begin{theorem}\label{the:stability}
Under the control law given in~\eqref{eqn:v}, the TVF described in~\eqref{eqn:model} asymptotically converges to the desired task configuration.
\end{theorem}
\begin{proof}
To prove Theorem~\ref{the:stability}, we conduct stability analysis of the TVF for each behavior, except the migration behavior $v_i^m$ due to its constant impact. 

\textit{Formation behavior:} Rewrite control law \eqref{eqn:uf} for the formation behavior as follows:
\begin{equation}
    v^f_i = k_f\sum_{j=1}^n{\left(p_j-p_i-\kappa \left(\delta_j^*-\delta_i^*\right)\right)}=k_f\sum_{j=1}^n{\left(p_j-p_i\right)}+b_i,
    \label{eq:dynamics}
\end{equation}
where $b_i=-k_f\sum_{j=1}^n\kappa\left(\delta_j^*-\delta_i^*\right)$ is the bias term.  Denote $H$ as the Laplacian matrix of the undirected sensing graph $\mathcal{G}$ of the swarm. As $\mathcal{G}$ is fully connected, element $h_{ij}$ of $H$ is given as follows:
\begin{equation}
    h_{ij}=\begin{cases}
    -1 & \text{if }i\neq j \\
    n-1 & \text{if }i=j
    \end{cases}
\end{equation}
With these values of its elements, the Laplacian matrix $H$ has one zero eigenvalue, while all other eigenvalues are identical and positive. Additionally, the bias vector $B=[b_1,...,b_n]^T$ is also an eigenvector of $H$ with corresponding eigenvalue $n$ so that
\begin{equation}
    HB=nB
\end{equation}

The swarm dynamics in (\ref{eq:dynamics}) then can be expressed as follows:
\begin{equation}
    \dot{P}=-k_fHP+B,
\end{equation}
where $P=\left[p_1,...,p_n\right]^T$. 

Define the Lyapunov-like function for the TVF system as follows:
\begin{equation}
    V_f=\dfrac{1}{2}\left(P-\dfrac{1}{k_fn}B\right)^T\left(P-\dfrac{1}{k_fn}B\right).
\end{equation}

Taking the first derivative of $V_f$ gives
\begin{equation}
\begin{aligned}
    \dot{V}_f&=\left(P-\dfrac{1}{k_fn}B\right)^T\left(-k_fHP+B\right)\\
    &=-k_f\left(P-\dfrac{1}{k_fn}B\right)^TH\left(P-\dfrac{1}{k_fn}B\right)\leq0.
\end{aligned}
\end{equation}

According to the LaSalle invariance
principle~\cite{9247488}, state $P$ will converge to the largest invariant subset $\Omega=\left\{P|\dot{V}_f=0\right\}$ as $k\to\infty$, which means the formation reaches the desired shape.

\textit{Tailgating behavior:} Let $l_i$ be the leader of robot $i$. Define a candidate Lyapunov function as follows:
\begin{equation}
    V_{t}=\dfrac{1}{2}\left(p_{l_i}-p_{i}-d_\text{ref}\dfrac{v_{l_i}}{\left\Vert v_{l_i}\right\Vert}\right)^{T}\left(p_{l_i}-p_{i}-d_\text{ref}\dfrac{v_{l_i}}{\left\Vert v_{l_i}\right\Vert}\right)
\end{equation}

Taking the first derivative of $V_t$ gives
\begin{equation}
    \dot{V}_{t}=\left(p_{l_i}-p_{i}-d_\text{ref}\dfrac{v_{l_i}}{\left\Vert v_{l_i}\right\Vert}\right)^{T}\left(v_{l_i}-v_{i}\right).
\end{equation}

By using the control law $v^t_i$ in \eqref{eqn:ut}, $\dot{V}_t$ becomes
\begin{equation}
    \dot{V}_{t}=-k_{t}\left(p_{l_i}-p_{i}-d_\text{ref}\dfrac{v_{l_i}}{\left\Vert v_{l_i}\right\Vert}\right)^{T}\left(p_{l_i}-p_{i}-d_\text{ref}\dfrac{v_{l_i}}{\left\Vert v_{l_i}\right\Vert}\right)\leq0.
    \label{eq:lyapunov}
\end{equation}
The Lyapunov stability condition is satisfied in (\ref{eq:lyapunov}). It indicates that the position $p_i$ of robot $i$ will align with its leader $p_{l_i}$ along the leader's direction $v_{l_i}$ and keep a distance $d_\text{ref}$ with it. This guarantees the straight line configuration.

\textit{Collision behaviors:} The collision avoidance behaviors activated when $\left\Vert p_{ij}\right\Vert<r_{a}$ for agents and $\left\Vert p_{io}\right\Vert<r_{a}$ for obstacles. The candidate Lyapunov functions for the robots within the alert range from those threats are defined as follows:
\begin{equation}
    V_i=\dfrac{1}{2}\sum_{j=1,j\neq i}^n\dfrac{1}{\left\Vert p_{ij}\right\Vert^2}
\end{equation}
\begin{equation}
    V_o=\dfrac{1}{2}\sum_{o\in\mathcal{M}_i}\dfrac{1}{\left\Vert p_{io}\right\Vert^2}
\end{equation}

Taking the first derivative of $V_i$ and $V_o$ gives
\begin{equation}
    \dot{V}_i=-\sum_{j=1,j\neq i}^n\dfrac{p_{ij}^T}{\left\Vert p_{ij}\right\Vert^4}(v_i-v_j)
\end{equation}
\begin{equation}
    \dot{V}_o=-\sum_{o\in\mathcal{M}_i}\dfrac{p_{io}^T}{\left\Vert p_{io}\right\Vert^4}v_i
\end{equation}

Substituting control laws $v^i_{ij}$ and $v^o_{io}$ in~\eqref{eqn:ui} -- \eqref{eqn:uo} to $\dot{V}_i$ and $\dot{V}_o$ gives
\begin{equation}
    \dot{V}_i=-k_i\sum_{j=1,j\neq i}^n\left(\dfrac{1}{\left\Vert p_{ij}\right\Vert }-\dfrac{1}{r_{a}}\right)\dfrac{1}{\left\Vert p_{ij}\right\Vert ^{5}}
\end{equation}
\begin{equation}
    \dot{V}_o=-k_o\sum_{o\in\mathcal{M}_i}\left(\dfrac{1}{\left\Vert p_{io}\right\Vert }-\dfrac{1}{r_{a}}\right)\dfrac{1}{\left\Vert p_{io}\right\Vert ^{5}}
\end{equation}
Since $\left\Vert p_{ij}\right\Vert<r_{a}$ and $\left\Vert p_{io}\right\Vert<r_{a}$, $\dot{V}_i <0$ and $\dot{V}_o<0$. The Lyapunov stability conditions for collision avoidance behaviors are satisfied. 

Next, summing the candidate Lyapunov functions in two separate modes, \textit{``formation''} and \textit{``tailgating''}, gives
\begin{equation}
\begin{aligned}
    V_F &= V_f+V_i+V_o\geq0\\
    V_T &= V_t+V_i+V_o\geq0
\end{aligned}
\end{equation}
and 
\begin{equation}
\begin{aligned}
    \dot{V}_F&=\dot{V}_f+\dot{V}_i+\dot{V}_o\leq0\\
    \dot{V}_T&=\dot{V}_t+\dot{V}_i+\dot{V}_o\leq0
\end{aligned}
\end{equation}
As a result, the system under control law~\eqref{eqn:v} in two separated modes, \textit{``formation''} and \textit{``tailgating''}, is asymptotically stable.
\end{proof}

\begin{remark}
The control architecture for the TVF in this work is decentralized. Each robot in the formation decides the mode itself based on information collected from its local sensors. The event-triggering condition in Algorithm~\ref{alg:our} is also distributed in the form of compact sets. 
\end{remark}

\begin{remark}
According to Lyapunov's theory, control law \eqref{eqn:v} ensures the system's asymptotic stability, such that under its action, the TVF will asymptotically achieve the desired configuration in both modes.
\end{remark}
\section{Results}\label{sec4}
A number of simulations and comparisons have been conducted to evaluate the performance of the proposed controller with details as follows.

\subsection{Evaluation setup}
\begin{figure*}
\begin{subfigure}[b]{\textwidth}
    
    \centering
    \includegraphics[width=0.9\textwidth]{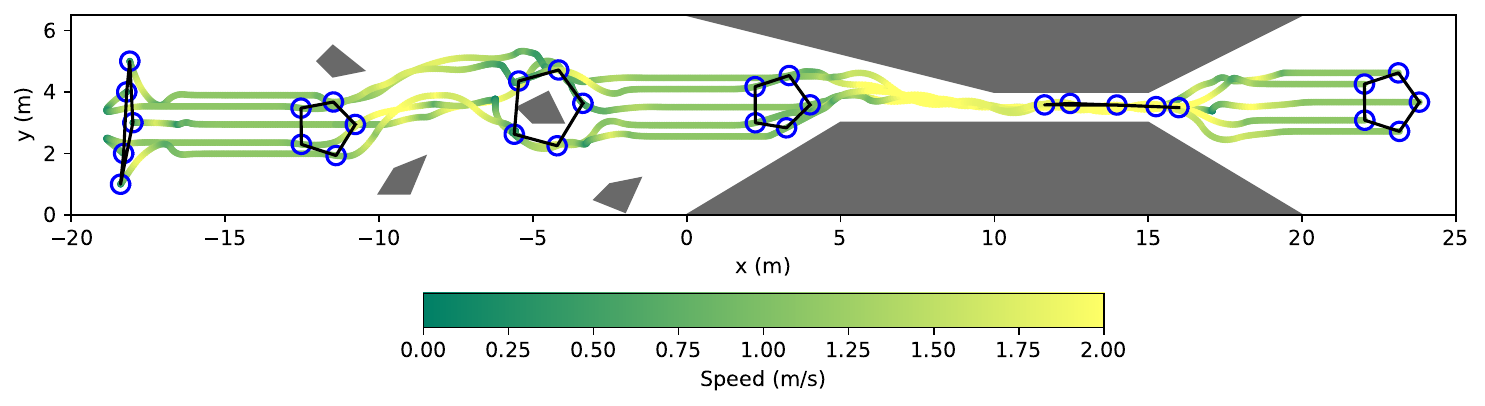}
    \caption{Pentagon configuration}
    \label{fig:path_erc1}
\end{subfigure}
\begin{subfigure}[b]{\textwidth}
    \centering
    \includegraphics[width=0.9\textwidth]{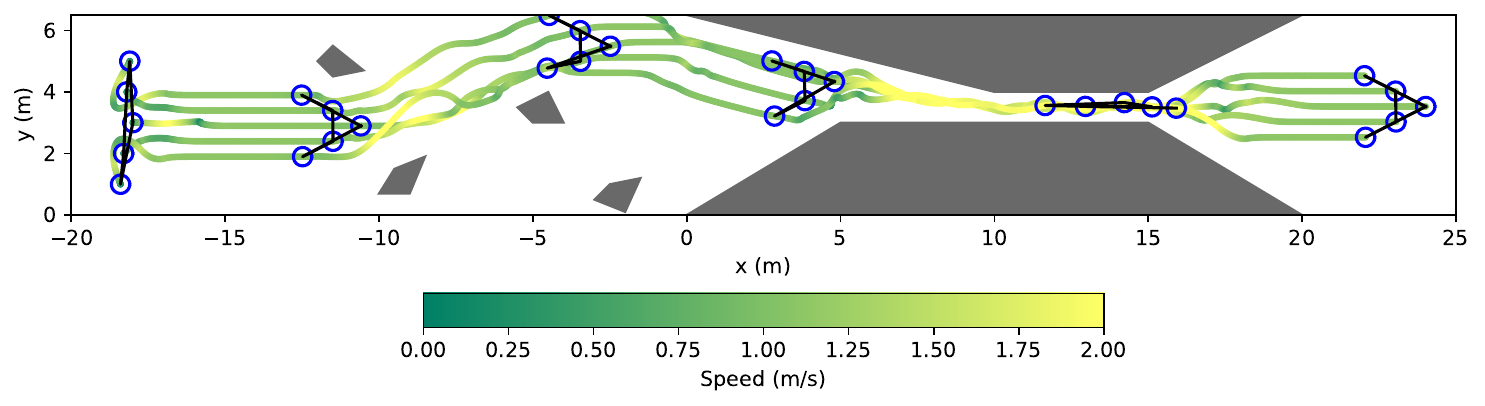}
    \caption{V-shape configuration}
    \label{fig:path_erc2}
\end{subfigure}
\caption{Motion paths and velocity profiles of the proposed ERC in the pentagon and V-shape configurations.}
\label{fig:path}
\end{figure*}

\begin{figure*}[!]
\begin{subfigure}[b]{0.49\textwidth}
    
    \centering
    \includegraphics[width=\linewidth]{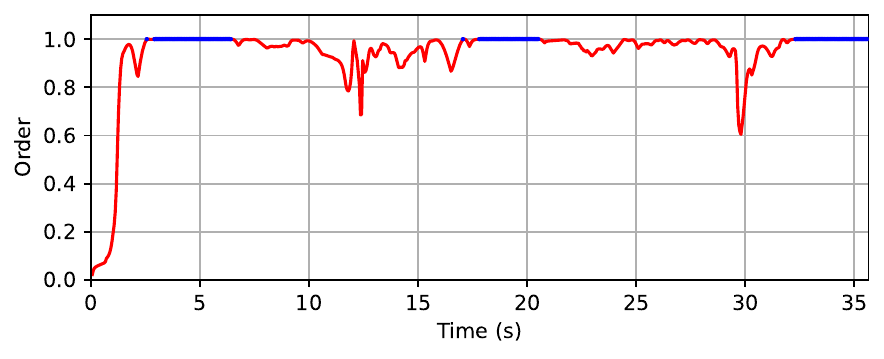}
    \caption{Pentagon configuration}
    \label{fig:order_erc1}
\end{subfigure}
\begin{subfigure}[b]{0.49\textwidth}
    \centering
    \includegraphics[width=\linewidth]{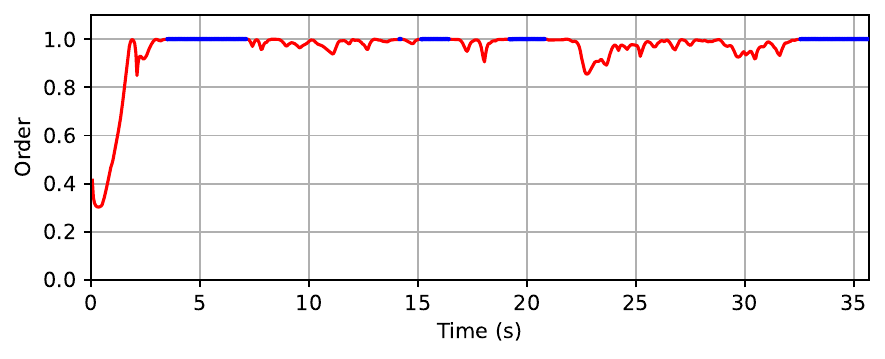}
    \caption{V-shape configuration}
    \label{fig:order_erc2}
\end{subfigure}
\caption{Order values of the proposed ERC}
\label{fig:order}
\end{figure*}

\begin{figure*}[!]
\begin{subfigure}[b]{0.49\textwidth}
    
    \centering
    \includegraphics[width=\linewidth]{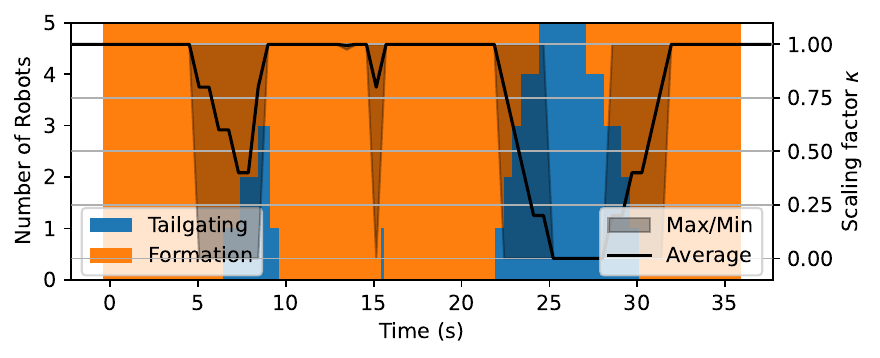}
    \caption{Pentagon configuration}
    \label{fig:mode_erc1}
\end{subfigure}
\begin{subfigure}[b]{0.49\textwidth}
    \centering
    \includegraphics[width=\linewidth]{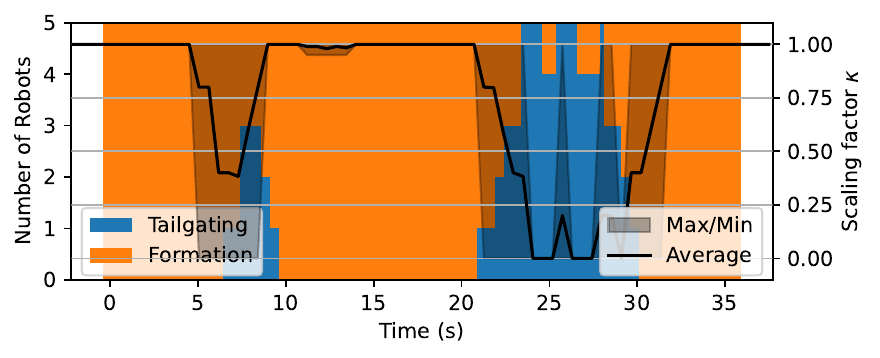}
    \caption{V-shape configuration}
    \label{fig:mode_erc2}
\end{subfigure}
\caption{Correlation between the robot number in each mode (bar chart) and the scaling factor (black line) of the proposed ERC}
\label{fig:mode}
\end{figure*}

\begin{table*}[!]
\caption{Comparison result between the APF, IAPF and ERC}
\label{tbl:com}
\centering
\begin{tabular}{C{1.6cm} C{1.cm} C{1.0cm} C{1.2cm} C{2.2cm} C{1.8cm} C{4.0cm}}
\hline\hline
Configu-ration             & Method & Success & Order $\Phi$  & Mean speed (m/s) & Travel time (s) & Mean energy consumption (m$^2$/s$^3$)  \\ \hline
\multirow{3}{*}{Pentagon} & APF    & 2/10    & 0.9435 & 0.8278           & 40.15           & 7432.3482                                       \\
                          & IAPF   & 6/10    & \textbf{0.9436} & 0.8520           & 39.35           & \textbf{6409.1530}
                          \\
                          & ERC    & \textbf{10/10}   & 0.9409 & \textbf{0.9803}           & \textbf{35.50}           & 7046.7390     \\ \hline
\multirow{3}{*}{V-shape}  & APF    & 3/10    & 0.9594 & 0.8095           & 40.30           & 7867.1044                                       \\
                              & IAPF   & 7/10    & \textbf{0.9597} & 0.8863           & 37.95           & 6317.7419  \\ 
                              & ERC    & \textbf{10/10}   & 0.9596 & \textbf{0.9519}           & \textbf{35.75}           & \textbf{6232.9238}   \\ \hline\hline
\multicolumn{7}{l}{\small *Bold values indicate the best outcomes.} \\
\end{tabular}
\end{table*}

\begin{figure*}[!]
    \centering
    \begin{subfigure}[b]{0.45\textwidth}
    \includegraphics[width=\textwidth]{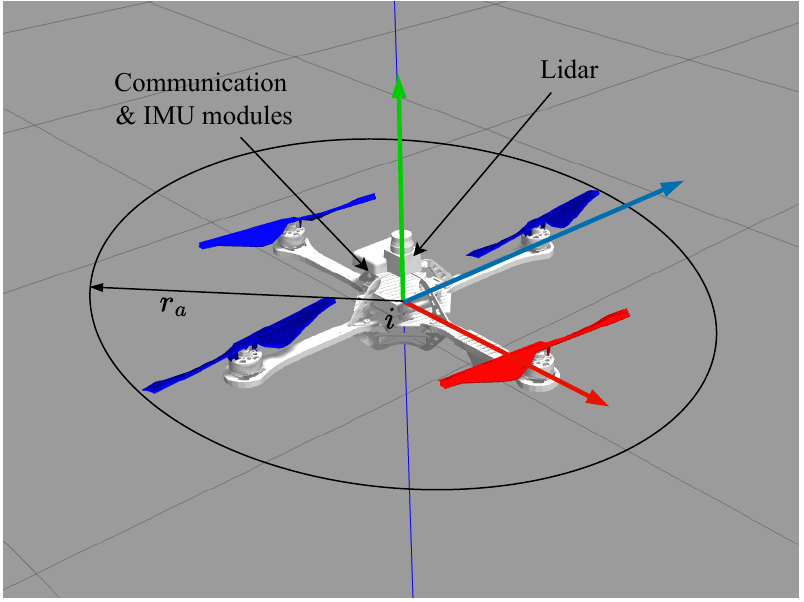}
    \caption{Hummingbird UAV model}
    \label{fig:gazebo_uav}
    \end{subfigure}
    \begin{subfigure}[b]{0.44\textwidth}
    \includegraphics[width=\textwidth,frame]{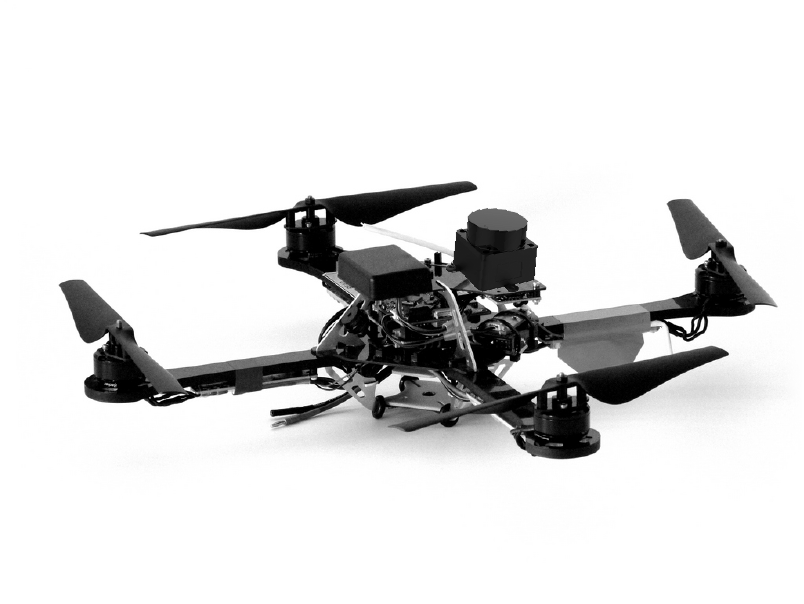}
    \caption{Real Hummingbird UAV \cite{Furrer2016,Bui2022}}
    \label{fig:gazebo_real}
    \end{subfigure}
    \caption{The Hummingbird UAV model used in SIL tests}
    \label{fig:gazebo_setup}
\end{figure*}

\begin{figure*}[!]
    \centering
    \begin{subfigure}[b]{0.495\textwidth}
    \includegraphics[width=\textwidth]{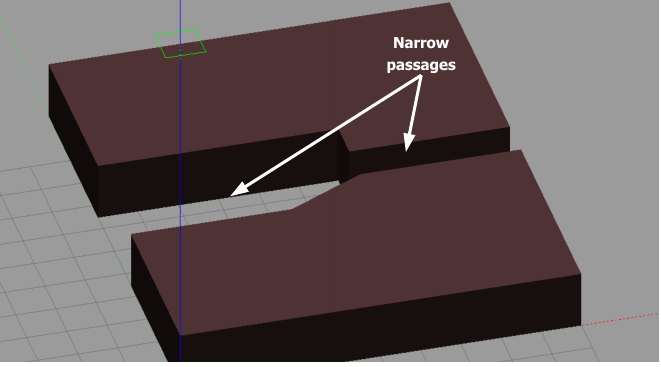}
    \caption{The environment}
    \end{subfigure}
    \begin{subfigure}[b]{0.40\textwidth}
    \includegraphics[width=\textwidth]{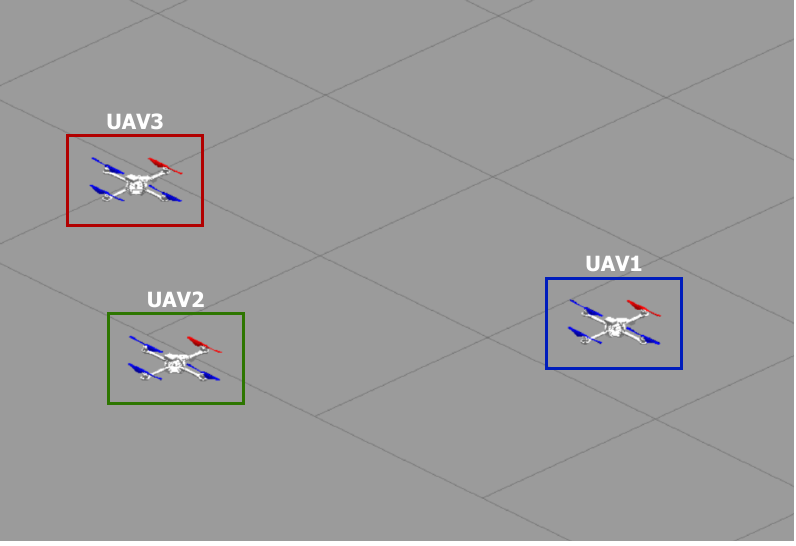}
    \caption{The UAVs at initial random positions}
    \label{fig:gazebo_init}
    \end{subfigure}
    \caption{Environment and UAV setup for SIL tests}
    \label{fig:gazebo_env}
\end{figure*}

The proposed controller is tested in a complex environment consisting of two areas with different types of obstacles. One area represents the forest-like environment with an obstacle density of 0.05 obs/m$^2$, whereas the other represents the cave-like environment with a narrow space of 1~m, as shown in Figure \ref{fig:path}. The TVF starts randomly from the left side of the environment and travels to the right side in the desired direction $u_\text{ref}=\left[1,0,0\right]^T$. The swarm includes 5 identical robots, each having the radius $r=0.3$~m, sensing range $r_s=3.0$~m, safe zone $r_a=3r$, maximum velocity $v_\text{max}=2$~m/s, and maximum acceleration $u_\text{max}=2$~m/s$^2$. The formation includes two task configurations, V-shape and polygon. In the \textit{``tailgating''} mode, the desired distance for a robot to its leader is set to $d_\text{ref}=1$~m. The control period is set at $\tau=0.1$~s. In evaluation, the order metric $\Phi$ is used to measure the correlation of the velocity vectors of the robots within the formation~\cite{9732989} and is computed as
\begin{equation}
    \Phi=\dfrac{1}{n(n-1)}\sum_{i=1}^n\sum_{j=1,j\neq i}^n  {\dfrac{\left\langle v_i,v_j\right\rangle}{\left\Vert v_i\right\Vert\left\Vert v_j\right\Vert}}.
\end{equation}
It has a range from 0 to 1. A value close to 1 means the robots have a high consensus on their direction~\cite{9732989,Vicsek1995}.

\subsection{Formation results}
Figure~\ref{fig:path} shows motion paths and velocity profiles of the TVF when navigating through the environment in the pentagon and V-shape configurations\footnote{Video: {\tt\url{https://youtu.be/XmZQIztNjn8}}}. All robots take off from the left side and then move to form the desired shape. They operate in the \textit{``formation''} mode if there is no narrow space in the area. Once narrow space is detected, the TVF transforms to the \textit{``tailgating''} mode to form a straight-line configuration. When the robots escape the narrow passage, they switch back to the \textit{``formation''} mode, which reforms to the desired configuration.

Figure~\ref{fig:order} shows the order metric of the TVF during the navigation process. In both scenarios, the swarm's order value remains close to 1, indicating a high level of directional consensus among the robots. The time intervals where the order value is less than 1 are highlighted in red color in Figure~\ref{fig:order}. The first segment on the left corresponds to the forming process from its initial state to the desired formation. The second segment in the middle presents the swarm navigating through the forest-like environment with obstacles. The third segment on the right depicts the swarm traversing narrow environments. The results also indicate that after passing through the forest-like and narrow environments, the swarm quickly transforms back to the task configuration, as shown in the blue segments with the order value of 1.


Figure~\ref{fig:mode} describes the correlation between the number of robots in different modes and the scaling factor $\kappa$. When the TVF encounters obstacles, the configuration is adapted based on the sensing data of each robot. As a result, the mode of each robot is different at the same time. The scaling factor $\kappa$ is also different among the robots depending on their position and sensing data. However, $\kappa=0$ when all robots in the TVF are in the \textit{``tailgating''} mode.


\subsection{Comparison results}
In another evaluation, comparisons with two other formation control methods, the pure behavior-based control using artificial potential field (APF)~\cite{Vsrhelyi2018} and the improved artificial potential field (IAPF)~\cite{9143127} are conducted. The evaluation metrics used in comparisons include the success rate, order, mean speed, travel time, and mean energy consumption. Each comparison is conducted over 10 simulations with 5 robots in two different configurations. 

The navigation results of the three methods are shown in Table~\ref{tbl:com}. The order is consistently high, approximating 1.0, for all methods implying that the UAVs are well aligned during operation. The proposed ERC, however, achieves the highest success rate, whereas the AFP and IAPF frequently fail due to collisions with narrow passages. The ERC also achieves the highest mean speed, the shortest travel time, and the lowest energy consumption for the V-shaped configuration. These results confirm the effectiveness of the proposed ERC in challenging scenarios.

\subsection{Validation with software-in-the-loop tests}

\begin{figure*}[!]
    \centering
    \begin{subfigure}[b]{0.48\textwidth}
    \includegraphics[width=\textwidth]{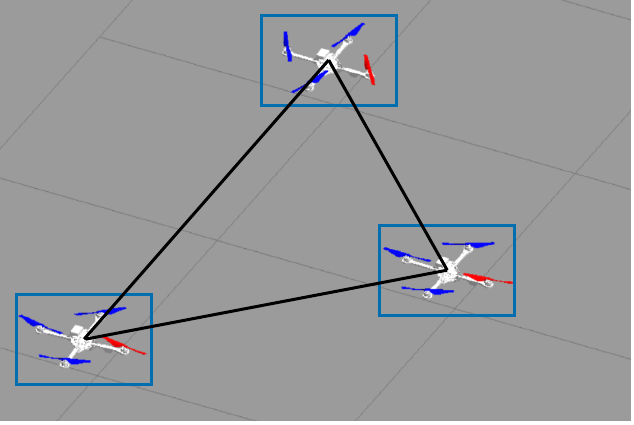}
    \caption{The desired triangular formation}
    \label{fig:gazebo_1}
    \end{subfigure}
    \begin{subfigure}[b]{0.48\textwidth}
    \includegraphics[width=\textwidth]{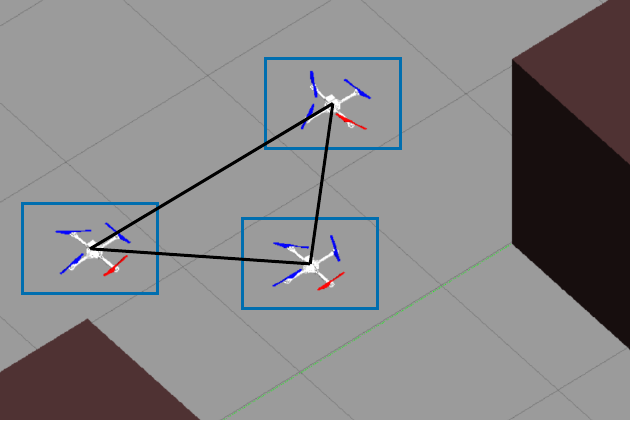}
    \caption{The scaled triangular formation}
    \label{fig:gazebo_2}
    \end{subfigure}
    \begin{subfigure}[b]{0.48\textwidth}
    \includegraphics[width=\textwidth]{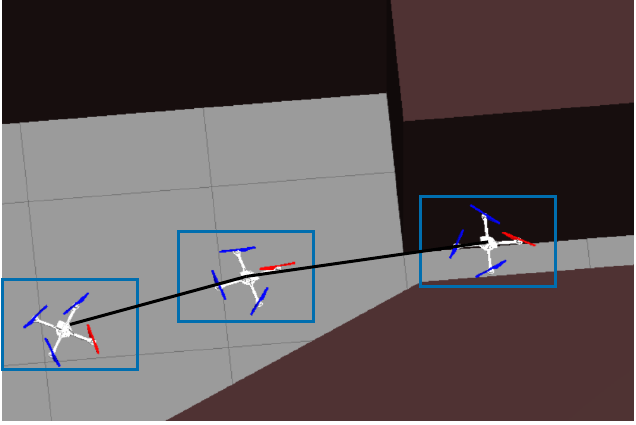}
    \caption{Transformation from the triangular to the line formation}
    \label{fig:gazebo_3}
    \end{subfigure}
    \begin{subfigure}[b]{0.48\textwidth}
    \includegraphics[width=\textwidth]{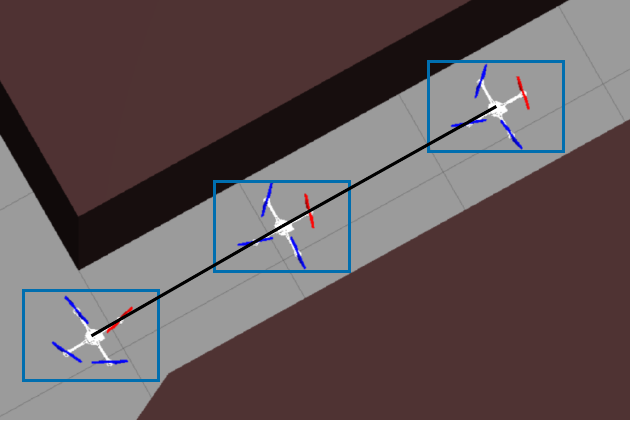}
    \caption{The line formation}
    \label{fig:gazebo_4}
    \end{subfigure}
    \begin{subfigure}[b]{0.48\textwidth}
    \includegraphics[width=\textwidth]{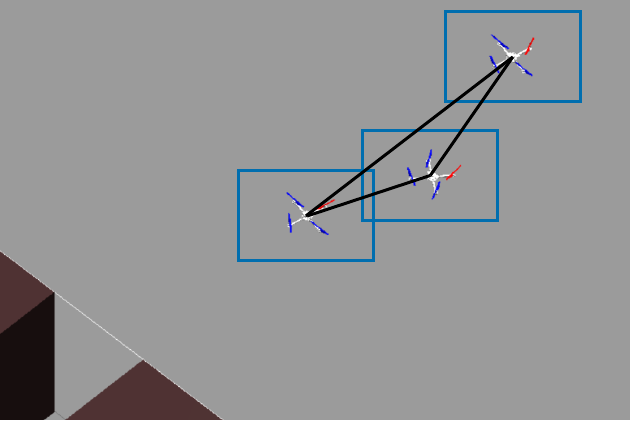}
    \caption{Transforming back to the triangular formation}
    \label{fig:gazebo_5}
    \end{subfigure}
    \begin{subfigure}[b]{0.48\textwidth}
    \includegraphics[width=\textwidth]{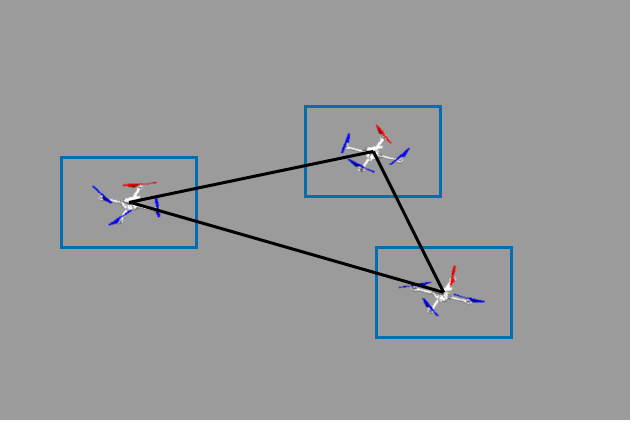}
    \caption{The original triangular formation}
    \label{fig:gazebo_6}
    \end{subfigure}
    \caption{Formation of the UAVs during the SIL test}
    \label{fig:gazebo_result}
\end{figure*}

\begin{figure*}[!]
    \centering
    \includegraphics[width=\textwidth]{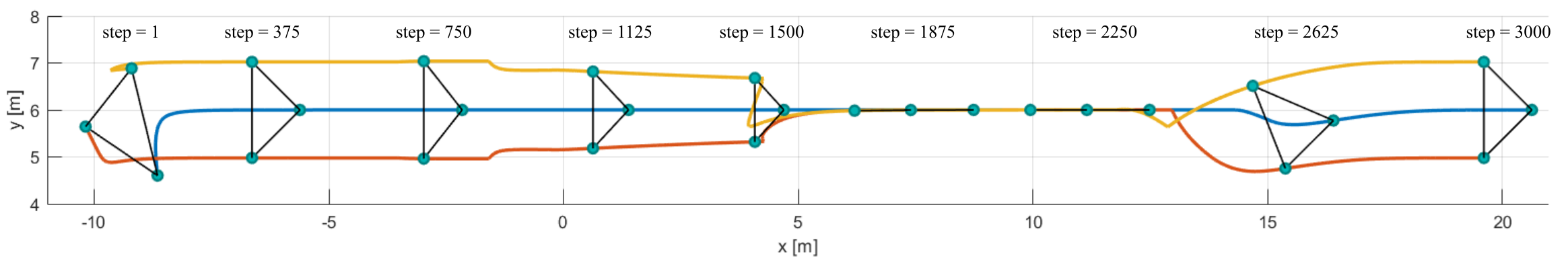}
    \caption{Trajectories of the UAVs during the SIL test}
    \label{fig:gazebo_path}
\end{figure*}

To further evaluate the validity of the proposed controller, software-in-the-loop (SIL) tests have been conducted using a Hummingbird quadrotor developed based on the Gazebo-based RotorS simulator~\cite{Furrer2016}, as described in Figure~\ref{fig:gazebo_setup}. The environment used in the SIL test includes two large obstacles, forming a width-varying tunnel, as depicted in Figure \ref{fig:gazebo_env}. 

Figures~\ref{fig:gazebo_result} and \ref{fig:gazebo_path} show the formation and paths of the UAVs in the SIL test. After taking off, the UAVs quickly form the desired triangular shape, as shown in Figure~\ref{fig:gazebo_1} and at steps 375 - 750 in Figure \ref{fig:gazebo_path}. When detecting the narrow passage, the formation shrinks to safely move through the passage (Figure~\ref{fig:gazebo_2} and steps 1125 - 1500 in Figure \ref{fig:gazebo_path}). When the passage is too narrow for the triangular formation, the UAVs switch to a straight line to pass through the space (Figure~\ref{fig:gazebo_3} - \ref{fig:gazebo_4} and steps 1875 - 2550 in Figure \ref{fig:gazebo_path}). Once passing the passage, the formation transforms back to the origin shape and moves to the target (Figure~\ref{fig:gazebo_5} - \ref{fig:gazebo_6} and steps 2625 - 3000 in Figure \ref{fig:gazebo_path}). The result hence confirms the validity of the proposed ERC in navigating the UAV formation through challenging environments.

\section{Conclusion}\label{sec5}
In this study, we have presented a novel approach called event-based reconfiguration control for managing the time-varying formation of a robot swarm navigating in challenging environments with narrow passages such as valleys, tunnels, and corridors. The controller features five behaviors designed using potential fields that can be combined into two control modes, \textit{formation} and \textit{tailgating}. An event-triggering mechanism utilizing sensing data from local sensors on each robot is introduced to determine the appropriate formation and control mode. This design ensures the system is distributed and scalable. The stability of the control system is validated through the Lyapunov theorem. Evaluation results show that the proposed controller effectively navigates the robot swarm in complex environments, outperforming established methods across most performance metrics. Software-in-the-loop tests further confirm the practicality and robustness of the proposed controller. Our future work will focus on robust reconfiguration formation control that deals with communication delays for complex collaborative tasks in clustered environments.

\section*{Acknowledgement}
Duy Nam Bui was funded by the Master, PhD Scholarship Programme of Vingroup Innovation Foundation (VINIF), code VINIF.2023.Ths.088.

\bibliographystyle{ieeetr}  
\bibliography{ref}

\end{document}